\documentclass{article}
\usepackage{iclr2027_conference,times}
\usepackage{amsmath,amssymb,amsthm}
\usepackage{graphicx}
\usepackage{booktabs}
\usepackage{hyperref}
\usepackage{url}

\newcommand{\method}{TreeSpark}
\newcommand{\ours}{\method}
\newcommand{\dspark}{DSpark}
\newcommand{\hdr}[2]{\begin{tabular}[c]{@{}c@{}}#1\\#2\end{tabular}}
\newcommand{\sg}{\operatorname{sg}}

\newcommand{\pa}{\operatorname{pa}}

\newcommand{\ind}{\mathbb{I}}
\newcommand{\E}{\mathbb{E}}
\newcommand{\cT}{\mathcal{T}}
\newcommand{\cF}{\mathcal{F}}
\newcommand{\cC}{\mathcal{C}}
\newcommand{\cL}{\mathcal{L}}
\newtheorem{proposition}{Proposition}

\title{TreeSpark: Learning to Draft the Trees You Verify for Speculative Decoding}

\author{%
\textbf{Jihao Zhao\textsuperscript{1},
Qianyu Jiang\textsuperscript{2},
Shuaishuai Zu\textsuperscript{1},
Pengfei Li\textsuperscript{3},}
Guanghuan Fang\textsuperscript{3},
Ben Feng\textsuperscript{3},\\
\textbf{Biao Qin\thanks{Corresponding authors}\hspace{0.6em}\textsuperscript{1},
Zhibin Wang\footnotemark[1]\hspace{0.6em}\textsuperscript{2},
Baochuan Yang\footnotemark[1]\hspace{0.6em}\textsuperscript{3}
}\\
\textsuperscript{1} School of Information, Renmin University of China\\
\textsuperscript{2} National Key Lab for Novel Software Technology, Nanjing University\\
\textsuperscript{3} Huawei Technologies Ltd.
}

\iclrfinalcopy
\begin{document}
\maketitle
\lhead{Preprint}

\begin{abstract}
Parallel drafters such as DFlash and DSpark reduce proposal latency, but early mismatches limit how much of each block can be accepted. Retaining alternatives introduces parent states and candidate choices that reference-conditioned training does not directly supervise. We introduce \method, a unified training and inference framework extending the semi-autoregressive DSpark architecture. It constructs parent-conditioned candidate trees from one backbone pass and replays their expansion, pruning, and budgeted packing during training. Survival negative log-likelihood weights retained edges by reference-derived path mass, while a capped coverage hinge promotes target-important alternatives near local selection boundaries. Both objectives reuse reference-chain target distributions without additional target forward passes. We evaluate greedy decoding with Qwen3-4B and Qwen3-8B on nine general benchmarks and six writing domains. Across targets and suites, \method improves mean acceptance length over DSpark by 23.4\% to 31.6\% and mean single-request throughput by 15.1\% to 25.0\%. Batched Qwen3-4B throughput improves by 12.6\% to 17.2\%. At fixed attention capacity and inference policy, the proposed supervision further improves mean acceptance length by approximately 6.4\% and mean system throughput by 4.2\% to 5.0\% on Qwen3-4B. Our code is available at \url{https://github.com/Robot2050/TreeSpark}.
\end{abstract}

\section{Introduction}
\label{sec:intro}

Speculative decoding reduces the sequential cost of language model inference by separating inexpensive proposal generation from parallel target verification \citep{xia2023specdec,leviathan2023fast,chen2023speculative}. Its efficiency depends on accepted length relative to drafting and verification costs \citep{xia2024survey}. Prior work develops feature-level autoregressive drafting \citep{li2024eagle} and blockwise parallel prediction \citep{stern2018blockwise}. DFlash combines target-feature conditioning with block diffusion to predict draft tokens in one backbone pass \citep{chen2026dflash}. DSpark retains parallel backbone computation and adds lightweight sequential correction for adjacent-token dependencies \citep{cheng2026dspark}. We study how coordinated training and inference can turn this inexpensive block computation into longer accepted continuations.

Low drafting latency alone does not ensure efficient use of a predicted block. With a single proposed continuation, the first rejection discards the remaining suffix regardless of its local prediction quality \citep{leviathan2023fast,chen2023speculative}. For a parallel drafter, computation has already been spent on every block position, but later predictions contribute only when the required prefix remains available. DSpark models adjacent-token dependencies, but selecting one continuation still excludes alternatives represented in its distributions. Its correction head is trained on reference parent tokens and encounters model-generated parents at deployment \citep{cheng2026dspark}. Accurate reference predictions therefore need not yield useful candidates under the deployed policy. Improving acceptance requires considering both the available continuations and the supervision that shapes them.

A budgeted candidate tree retains alternatives while reusing parallel computation. Tree verification evaluates continuations with shared prefixes in one target pass \citep{miao2024specinfer,cai2024medusa}, with candidate allocation constrained by verification cost \citep{chen2024sequoia}. In the semi-autoregressive architecture, shared backbone outputs support lightweight parent-dependent corrections. This extension creates a specific learning requirement: each depth-specific prediction must support several proposed parents, and only a subset of their children receives verification positions. Reference-token losses do not directly supervise this allocation. Useful alternatives may fall below local selection boundaries or lose their prefixes during pruning. Training must therefore account for both local prediction quality and the selection decisions that determine complete-path availability.

Draft-generated data and online distillation improve alignment with inference-time continuations and query distributions \citep{zhou2024distillspec,liu2024online}. HASS and EAGLE-3 expose training to simulated inference states \citep{zhang2025hass,li2025eagle3}, while LK losses optimize speculative acceptance directly \citep{samarin2026lk}. Our focus is the additional alignment required when a semi-autoregressive drafter retains multiple continuations. The correction head must learn on proposed parents, reinforce retained paths, and promote useful alternatives under the deployed budget. Sharing the backbone and reusing reference target distributions make this feasible without repeated transformer computation or new target evaluations on every branch.

We introduce \method, which integrates parent-conditioned candidate generation with training under the deployed policy (Figure~\ref{fig:framework}). Training replays expansion, pruning, and packing with the current drafter while retaining reference distillation. Survival weights retained-edge likelihood by reference-derived path mass; capped coverage promotes target-important alternatives near local selection boundaries, including omitted candidates. Both losses update the shared backbone and Markov head on model-generated parents, with discrete selection held fixed. Reference target distributions are reused across branches. Deployment follows the same candidate policy with greedy verification and one draft backbone pass per round. The deployment policy defines the states and selections used for supervision. Learning can therefore improve the candidates presented to a fixed verifier without introducing a separate scoring model at inference.

The main contributions are:
\begin{itemize}
 \item A unified extension of semi-autoregressive drafting that retains alternative continuations with one backbone pass and trains on model-generated parent states by replaying the deployed expansion, pruning, and packing policy.
 \item Complementary survival and coverage objectives that weight retained-edge likelihood by reference-derived path mass and promote useful alternatives near local selection boundaries, reusing reference target distributions without additional target forward passes.
 \item Evaluations against DFlash and DSpark on two Qwen3 targets, with mean single-request throughput gains of 15.1\% to 25.0\% over DSpark. Controls at fixed attention capacity and inference policy isolate the additional benefit of supervision.
\end{itemize}

\section{Related Work}
\label{sec:related}

\paragraph{Parallel and semi-autoregressive drafting.}
Blockwise parallel decoding predicts several positions before validating their longest usable prefix \citep{stern2018blockwise}. EAGLE instead drafts target features autoregressively \citep{li2024eagle}. DFlash combines target-feature conditioning with a lightweight block diffusion drafter, reducing proposal generation to one backbone pass \citep{chen2026dflash}. DSpark augments this structure with sequential token correction and confidence-based verification scheduling \citep{cheng2026dspark}. These methods motivate our focus on increasing accepted length while retaining parallel backbone computation. \method\ builds on the DSpark Markov-head architecture and couples the generation of alternative continuations with supervision on the parent states and candidate selections used at deployment.

\paragraph{Training and inference alignment.}
DistillSpec studies draft-generated data and decoding-specific distillation \citep{zhou2024distillspec}, while Online Speculative Decoding updates the drafter from verification feedback \citep{liu2024online}. HASS aligns training with speculative rollout \citep{zhang2025hass}, EAGLE-3 simulates inference during training \citep{li2025eagle3}, and LK losses directly optimize speculative acceptance \citep{samarin2026lk}. DFlash also aligns masked-block construction with its inference inputs. \method\ extends alignment to the candidate decisions introduced by branching: survival supervises retained paths, and coverage addresses alternatives near the admission boundary. Both objectives reuse reference target distributions while training the parallel backbone and correction head on proposed parent states.

\paragraph{Candidate construction and verification.}
SpecInfer and Medusa establish efficient verification of multiple continuations \citep{miao2024specinfer,cai2024medusa}. EAGLE-2 and OPT-Tree adapt candidate allocation \citep{li2024eagle2,wang2025opttree}; Sequoia and CaDDTree account for verification efficiency and cost \citep{chen2024sequoia,zhang2026caddtree}. DDTree constructs trees from block-diffusion marginals, TAPS learns a target-aware scorer for prefix selection, and JetSpec trains a causal parallel drafter by distillation \citep{ringel2026ddtree,wang2026taps,hu2026jetspec}. \method\ retains the DSpark semi-autoregressive architecture and replays deployment selections during training, coupling likelihood supervision on retained paths with local admission signals for useful alternatives.

\section{Aligning Parallel Drafting with Verification}
\label{sec:method}

\begin{figure}[t]
 \centering
 \includegraphics[width=\linewidth]{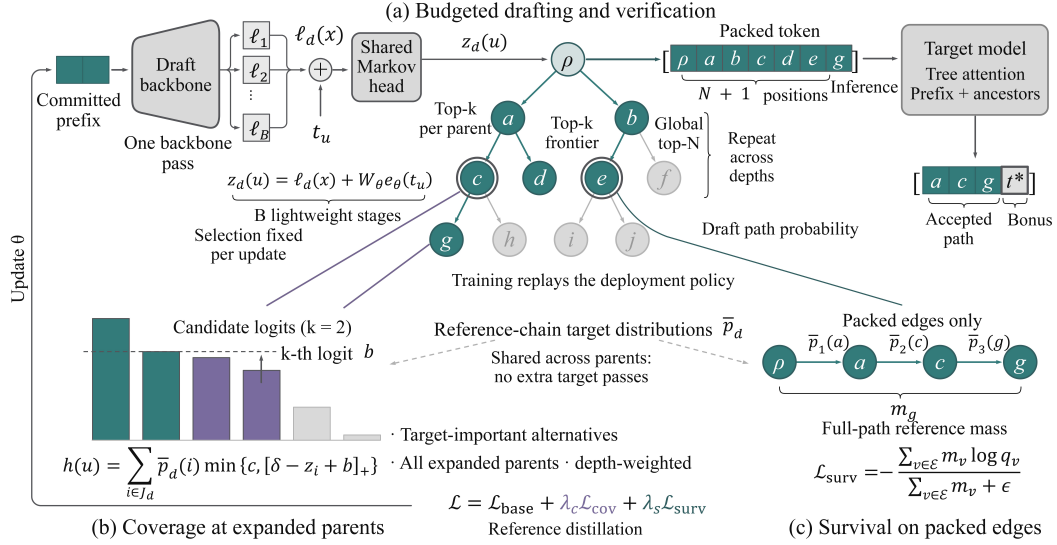}
 \caption{\textbf{\method\ framework.} (a) Budgeted drafting and verification. (b) Coverage at expanded parents. (c) Survival on valid packed edges, weighted by reference path mass. Both objectives reuse reference target distributions. Teal nodes are packed; double outlines mark the frontier. The schematic uses $k=2$; experiments use $k=4$. Panel (c) illustrates one supervised path. Equation~\ref{eq:draft} defines parent conditioning. The feedback loop updates draft parameters; global packing follows depthwise expansion.}
 \label{fig:framework}
\end{figure}

\method\ extends the DSpark parallel backbone and lightweight correction head through coordinated training and inference (Figure~\ref{fig:framework}). Inference retains alternative continuations in a bounded tree, and training replays the same candidate policy on model-generated parent states. Survival supervises retained edges, while coverage addresses local candidate admission before packing. The joint objective updates the backbone and Markov head with discrete selection held fixed. We first define candidate generation, then describe its replay and supervision.

\subsection{Parent-Conditioned Tree Drafting}
\label{sec:tree}

Let $x$ denote the committed prefix, whose final token is the root anchor $\rho$, and let $P$ denote the frozen target model. We instantiate the semi-autoregressive drafter with the DSpark backbone and Markov head, applying the correction separately to each frontier parent. A single backbone pass produces vocabulary logits $\ell_d(x)$ for each draft depth $d\in\{1,\ldots,B\}$. For a parent $u$ carrying token $t_u$, a rank-$r$ correction yields
\begin{equation}
 z_d(u)=\ell_d(x)+W_\theta e_\theta(t_u),
 \qquad
 q_{\theta,d}(\cdot\mid u,x)=\operatorname{softmax}\!\left(z_d(u)\right).
 \label{eq:draft}
\end{equation}
Here $\mathcal V$ is the vocabulary and $\theta$ collects trainable draft parameters. The Markov embedding $e_\theta(t_u)\in\mathbb{R}^r$ and projection $W_\theta\in\mathbb{R}^{|\mathcal V|\times r}$ are distinct from the frozen backbone input embeddings and output projection. All parents at a given depth share the backbone prediction but receive token-dependent corrections. Branching therefore requires additional lightweight head evaluations without repeating the backbone.

Tree construction uses a local branching factor $k$, frontier width $k$, and budget $N$ of non-root nodes. Starting from the anchor, each frontier parent proposes its top-$k$ children. Every child enters a candidate pool and receives the product of draft conditional probabilities along its path as its score. The $k$ highest-scoring children at that depth form the next frontier. After $B$ depths, global packing retains up to $N$ highest-scoring nodes from the pool, preferring shallower nodes when scores tie. Frontier pruning bounds expansion cost, whereas global packing determines the input to target verification.

The resulting tree contains every ancestor of each selected node: a child has no larger path probability than its parent, and the tie rule places the parent first. Tree attention then lets each node attend to the committed prefix and its own ancestor path. Greedy verification follows children that match the target argmax and emits a target-selected bonus token when the matching path ends. Under identical per-path target evaluations and argmax tie rules, this procedure preserves autoregressive greedy output for any such tree. Appendix~\ref{app:verification} gives the construction and correctness arguments.

\subsection{Training under the Deployment Policy}
\label{sec:replay}

Reference training provides anchor contexts, continuations $y_{1:B}$, and aligned target distributions $\bar p_d=P(\cdot\mid x,y_{<d})$. We retain the baseline reference losses and replay the tree policy on a subset of training anchors with matching $B$, $k$, and $N$. The replay uses the current draft parameters, so each correction head is evaluated on model-selected parent tokens. The selected subtree consequently reflects both the model state and the verification budget.

Replay separates discrete selection from differentiable prediction. Token identities, frontier choices, packing decisions, and supervision weights are detached. The conditional probabilities and logits used in the losses remain differentiable through the backbone and correction head. Each subsequent update reconstructs the candidate trees. This yields a direct supervised update conditional on the current tree selection.

All parents at depth $d$ reuse $\bar p_d$, obtained by applying softmax to the aligned reference target logits. These labels are exact for the reference context and serve as proxies on alternative branches, avoiding additional target evaluations. Replay therefore matches the parent-conditioned expansion and budget policy while retaining reference-based target supervision. Numerical selection details, including tie handling, are specified in Appendix~\ref{app:numerics}; the full training procedure appears in Appendix~\ref{app:training}.

\subsection{Survival and Coverage Supervision}
\label{sec:objectives}

\paragraph{Survival on the selected subtree.}
Useful depth requires an entire proposed prefix to remain available. We therefore weight a retained edge by the reference-derived probability mass of its full path. For node $v$, let $m_v$ be the product of reference target probabilities along its proposed path, excluding the already committed anchor, and let $q_v$ denote its draft conditional probability from Equation~\ref{eq:draft}. If $\mathcal E$ collects valid non-root nodes packed across the replayed anchors, the survival negative log-likelihood is
\begin{equation}
 \cL_{\mathrm{surv}}
 =-\frac{\sum_{v\in\mathcal E}m_v\log q_v}
          {\sum_{v\in\mathcal E}m_v+\epsilon}.
 \label{eq:survival}
\end{equation}
The weights $m_v$ are fixed during differentiation, and $\epsilon>0$ stabilizes the denominator. Path weighting reduces the influence of descendants reached through prefixes with little reference target mass, while the packing restriction focuses explicit likelihood supervision on candidates that occupy verification slots. Normalization by total mass keeps the objective on a comparable scale across replayed trees.

For a fixed selected child set, this loss moves the parent distribution toward the relative reference mass of those children. It supplies no positive likelihood target for an omitted child, even when the target assigns that token substantial probability. Coverage supervision addresses this missing signal. Appendix~\ref{app:analysis} derives the local gradient and relates path mass to the expected depth covered by a fixed tree. This relation motivates the weighting; the implemented loss remains a surrogate for greedy acceptance.

\paragraph{Coverage near the local selection boundary.}
At an expanded parent $u$, let $b_d(u)$ be the $k$th largest draft logit and let $J_d$ contain the $J$ most probable tokens under $\bar p_d$. We define
\begin{equation}
 h_d(u)=\sum_{i\in J_d}\bar p_d(i)
 \min\!\left\{c,\,
       \left[\delta-z_d(u)_i+b_d(u)\right]_+\right\},
 \label{eq:hinge}
\end{equation}
where $[a]_+=\max(0,a)$, $\delta>0$ is a margin, and $c>0$ caps the penalty. The boundary measures whether a token favored by the target can compete for a local candidate slot. Within the active margin, the loss penalizes insufficient separation from that boundary without specifying an ordering among the target top-$J$ tokens. The cap limits the influence of large violations, concentrating gradients in a bounded neighborhood of the selection threshold.

We aggregate $h_d(u)$ over all valid expanded parents, with exponential depth decay, and normalize by the correspondingly weighted target top-$J$ mass to obtain $\cL_{\mathrm{cov}}$. Differentiation includes both the token logit and boundary logit. Appendix~\ref{app:training} gives the full reduction; Appendix~\ref{app:analysis} characterizes the active gradient interval and the cancellation when a supervised token itself defines the boundary.

The objectives operate at complementary stages of selection. Coverage supervises expanded parents before the final node budget is applied, allowing omitted alternatives to receive a training signal when they fall within the active margin. Survival supervises the edges retained after global packing. Their combination connects local candidate admission with likelihood learning on the subtree actually allocated verification positions.

\subsection{Training and Inference}
\label{sec:execution}

The complete objective augments reference distillation:
\begin{equation}
 \cL=\cL_{\mathrm{base}}
     +\lambda_s\cL_{\mathrm{surv}}
     +\lambda_c\cL_{\mathrm{cov}}.
 \label{eq:total}
\end{equation}
The baseline combines reference-token cross-entropy, distributional $L_1$ matching, and the inherited confidence loss. These terms retain supervision beyond the currently selected branches. The confidence head is auxiliary to training and does not control the fixed tree budget. We also record mass coverage and reference argmax path survival on the packed tree to inspect how its candidates change during training; their definitions are provided in Appendix~\ref{app:diagnostics}.

Our configuration uses $B=7$, $k=4$, and $N=32$, giving up to 100 candidates before packing and at most 33 target input positions including the anchor. Replay uses up to 128 anchors per sequence, with checkpointing to control activation memory. Inference batches frontier expansion and verification across requests, retaining one draft backbone pass followed by $B$ lightweight correction stages. The tree objectives add no target evaluations during training or model components at inference.

Tree replay requires a shared backbone prediction to support several parent-conditioned branch decisions. To increase the representational capacity available for learning these decisions, we widen the draft attention head dimension from 128 to 192 while retaining the layer count and tree policy. This expands the backbone attention projections; the Markov correction rank remains unchanged. Section~\ref{sec:ablation} evaluates this capacity choice separately from the two objectives, and Appendix~\ref{app:implementation} specifies its computational implications.

\section{Experiments}
\label{sec:experiments}

We evaluate whether the proposed training and inference extension improves acceptance and throughput over DFlash and DSpark. A reference-trained branching control and component ablations examine the contributions of deployment-aligned supervision and attention capacity.

\begin{table}[t]
\caption{Single-request throughput (tokens/s): (a) reasoning, code, and dialogue; (b) writing. Vanilla denotes autoregressive decoding. Mean averages benchmarks. $\Delta$ compares each benchmark with DSpark; its Mean entry averages these gains. Bold marks column maxima.}
\label{tab:throughput-c1}
\centering
\footnotesize
\setlength{\tabcolsep}{2pt}
\begin{tabular*}{\textwidth}{@{\extracolsep{\fill}}lcccccccccc@{}}
\toprule
\multicolumn{11}{@{}l}{\emph{(a) Reasoning, code, and dialogue}} \\
\midrule
 & GSM8K & \hdr{MATH}{500} & \hdr{AIME}{25} & \hdr{Human}{Eval} & MBPP & \hdr{LiveCode}{Bench} & \hdr{MT-}{Bench} & Alpaca & \hdr{Arena}{Hard-v2} & Mean \\
\midrule
\multicolumn{11}{@{}l}{\emph{Qwen3-4B}} \\
Vanilla & 12.83 & 11.66 & 10.82 & 12.80 & 13.13 & 11.53 & 12.99 & 12.78 & 11.72 & 12.25 \\
EAGLE-3 & 25.20 & 23.27 & 21.29 & 20.75 & 21.19 & 22.89 & 16.06 & 15.51 & 16.03 & 20.24 \\
DFlash & 46.66 & 43.48 & 37.48 & 35.36 & 36.59 & 42.88 & 24.21 & 22.87 & 23.58 & 34.79 \\
\dspark{} & 49.74 & 47.29 & 41.22 & 39.39 & 37.22 & 42.06 & 24.30 & 23.06 & 24.05 & 36.48 \\
\ours{} & \textbf{55.58} & \textbf{52.12} & \textbf{48.21} & \textbf{45.38} & \textbf{44.06} & \textbf{46.43} & \textbf{28.65} & \textbf{27.16} & \textbf{28.32} & \textbf{41.77} \\
$\Delta$ (\%) & +11.7 & +10.2 & +17.0 & +15.2 & +18.4 & +10.4 & +17.9 & +17.8 & +17.8 & +15.1 \\
\midrule
\multicolumn{11}{@{}l}{\emph{Qwen3-8B}} \\
Vanilla & 12.80 & 12.29 & 11.74 & 13.56 & 13.33 & 12.03 & 13.30 & 13.20 & 12.05 & 12.70 \\
EAGLE-3 & 38.53 & 36.93 & 32.35 & 30.33 & 29.15 & 34.65 & 19.28 & 18.02 & 18.98 & 28.69 \\
DFlash & 45.68 & 42.43 & 36.03 & 34.67 & 33.93 & 39.80 & 22.22 & 20.88 & 22.39 & 33.11 \\
\dspark{} & 47.94 & 45.42 & 40.34 & 38.75 & 38.20 & 43.03 & 24.54 & 22.40 & 23.38 & 36.00 \\
\ours{} & \textbf{57.89} & \textbf{53.47} & \textbf{46.58} & \textbf{46.88} & \textbf{45.12} & \textbf{47.53} & \textbf{29.45} & \textbf{26.93} & \textbf{27.38} & \textbf{42.36} \\
$\Delta$ (\%) & +20.8 & +17.7 & +15.5 & +21.0 & +18.1 & +10.5 & +20.0 & +20.2 & +17.1 & +17.9 \\
\bottomrule
\end{tabular*}
\vspace{2mm}

\begin{tabular*}{\textwidth}{@{\extracolsep{\fill}}lccccccc@{}}
\toprule
\multicolumn{8}{@{}l}{\emph{(b) Writing}} \\
\midrule
 & \hdr{Academic}{\& Engineering} & \hdr{Advertising}{\& Marketing} & Education & \hdr{Finance}{\& Business} & \hdr{Literature}{\& Arts} & \hdr{Politics}{\& Law} & Mean \\
\midrule
\multicolumn{8}{@{}l}{\emph{Qwen3-4B}} \\
Vanilla & \phantom{0}9.81 & \phantom{0}9.69 & \phantom{0}9.74 & 10.16 & 10.20 & 10.28 & \phantom{0}9.98 \\
EAGLE-3 & 14.39 & 13.57 & 14.22 & 15.22 & 14.16 & 15.30 & 14.48 \\
DFlash & 20.24 & 18.61 & 19.77 & 21.28 & 18.54 & 20.46 & 19.82 \\
\dspark{} & 22.45 & 20.15 & 21.88 & 23.36 & 19.48 & 23.07 & 21.73 \\
\ours{} & \textbf{27.96} & \textbf{24.80} & \textbf{26.80} & \textbf{28.68} & \textbf{23.69} & \textbf{26.96} & \textbf{26.48} \\
$\Delta$ (\%) & +24.5 & +23.1 & +22.5 & +22.8 & +21.6 & +16.9 & +21.9 \\
\midrule
\multicolumn{8}{@{}l}{\emph{Qwen3-8B}} \\
Vanilla & 10.28 & 10.21 & 10.23 & 10.37 & 10.76 & 10.79 & 10.44 \\
EAGLE-3 & 17.78 & 16.64 & 17.78 & 19.29 & 16.22 & 18.35 & 17.68 \\
DFlash & 19.43 & 18.03 & 19.95 & 21.14 & 17.46 & 19.78 & 19.30 \\
\dspark{} & 21.65 & 19.89 & 21.30 & 22.61 & 18.58 & 21.04 & 20.84 \\
\ours{} & \textbf{27.07} & \textbf{24.78} & \textbf{26.90} & \textbf{28.41} & \textbf{22.94} & \textbf{26.26} & \textbf{26.06} \\
$\Delta$ (\%) & +25.0 & +24.6 & +26.3 & +25.7 & +23.5 & +24.8 & +25.0 \\
\bottomrule
\end{tabular*}
\end{table}

\begin{table}[t]
\caption{Acceptance length on reasoning, code, and dialogue tasks, including the target bonus when emitted. All methods allow seven draft depths; \method\ packs 32 non-root nodes. Mean averages benchmarks; its $\Delta$ compares the two means for \method\ and DSpark. Other $\Delta$ entries compare individual benchmarks. Bold marks column maxima. Writing results appear in Table~\ref{tab:writing-acceptance}.}
\label{tab:acceptance-length}
\centering
\footnotesize
\setlength{\tabcolsep}{2pt}
\begin{tabular*}{\textwidth}{@{\extracolsep{\fill}}lcccccccccc@{}}
\toprule
\multicolumn{11}{@{}l}{\emph{Reasoning, code, and dialogue}} \\
\midrule
 & GSM8K & \hdr{MATH}{500} & \hdr{AIME}{25} & \hdr{Human}{Eval} & MBPP & \hdr{LiveCode}{Bench} & \hdr{MT-}{Bench} & Alpaca & \hdr{Arena}{Hard-v2} & Mean \\
\midrule
\multicolumn{11}{@{}l}{\emph{Qwen3-4B}} \\
EAGLE-3 & 3.43 & 3.25 & 3.04 & 2.83 & 2.76 & 3.20 & 2.13 & 2.08 & 2.20 & 2.77 \\
DFlash & 4.97 & 4.95 & 4.42 & 3.91 & 3.76 & 4.82 & 2.52 & 2.41 & 2.59 & 3.82 \\
\dspark{} & 5.46 & 5.44 & 4.97 & 4.41 & 4.22 & 5.24 & 2.80 & 2.63 & 2.90 & 4.23 \\
\ours{} & \textbf{6.57} & \textbf{6.58} & \textbf{6.26} & \textbf{5.54} & \textbf{5.40} & \textbf{6.23} & \textbf{3.59} & \textbf{3.42} & \textbf{3.64} & \textbf{5.25} \\
$\Delta$ (\%) & +20.3 & +21.0 & +26.0 & +25.6 & +28.0 & +18.9 & +28.2 & +30.0 & +25.5 & +24.1 \\
\midrule
\multicolumn{11}{@{}l}{\emph{Qwen3-8B}} \\
EAGLE-3 & 4.86 & 4.78 & 4.21 & 3.76 & 3.58 & 4.65 & 2.22 & 2.03 & 2.24 & 3.59 \\
DFlash & 4.96 & 4.87 & 4.33 & 3.91 & 3.73 & 4.80 & 2.48 & 2.34 & 2.54 & 3.77 \\
\dspark{} & 5.53 & 5.41 & 4.89 & 4.39 & 4.20 & 5.24 & 2.77 & 2.57 & 2.80 & 4.20 \\
\ours{} & \textbf{6.61} & \textbf{6.56} & \textbf{5.96} & \textbf{5.57} & \textbf{5.35} & \textbf{6.27} & \textbf{3.55} & \textbf{3.32} & \textbf{3.47} & \textbf{5.18} \\
$\Delta$ (\%) & +19.5 & +21.3 & +21.9 & +26.9 & +27.4 & +19.7 & +28.2 & +29.2 & +23.9 & +23.4 \\
\bottomrule
\end{tabular*}
\end{table}

\subsection{Experimental Setup}
\label{sec:setup}

We use Qwen3-4B and Qwen3-8B targets \citep{yang2025qwen3} with greedy decoding. All drafters are trained for ten epochs on a common mixture of 50,000 prompts covering mathematics, code, dialogue, and writing, with a maximum training sequence length of 4,096 tokens. Evaluation retains complete input prompts and generates up to 2,048 tokens. The general suite comprises nine reasoning, coding, and dialogue benchmarks; WritingBench supplies six writing domains \citep{wu2025writingbench}. Writing training prompts come from Writer-R1 \citep{zhao2026writerr1} and are disjoint from WritingBench. Appendix~\ref{app:experimental-protocol} details the datasets and training settings.

DFlash and DSpark are the primary parallel-drafting baselines; autoregressive decoding and EAGLE-3 provide additional reference points. Following the DSpark drafter-depth configuration, EAGLE-3 uses one layer, while DFlash, DSpark, and \method\ use five. In our evaluation, EAGLE-3, DFlash, and DSpark each produce a single proposal path per round. All speculative methods allow up to seven draft depths. \method\ uses branching factor four and a budget of 32 non-root nodes, giving up to 33 target input positions per tree, compared with eight for a chain. Experiments run on Ascend 910B accelerators: one card for total concurrency $C\in\{1,6,12\}$, and three independent serving workers with 12 requests each for $C=36$.

Acceptance length counts committed tokens per verification round, including the target bonus when emitted. Because batching and KV-cache utilization also affect serving efficiency \citep{kwon2023pagedattention}, we measure system throughput as aggregate output tokens per second. Throughput results average three repeated measurements. Benchmarks receive equal weight. Throughput gains over chain DSpark average benchmark-specific percentage improvements; acceptance gains and throughput comparisons between tree configurations use ratios of suite means. Appendix~\ref{app:experimental-metrics} specifies the metrics and aggregation.

\subsection{Acceptance across Tasks and Draft Positions}
\label{sec:acceptance-results}

\method\ achieves the highest acceptance length on all nine general benchmarks for both targets (Table~\ref{tab:acceptance-length}). The mean rises from 4.23 to 5.25 for Qwen3-4B and from 4.20 to 5.18 for Qwen3-8B, improving by 24.1\% and 23.4\%, respectively. Relative gains are larger on writing: 2.68 to 3.52 and 2.55 to 3.35, corresponding to 31.4\% and 31.6\%; Appendix~\ref{app:writing-results} provides the complete domain breakdown. The improvement therefore extends to workloads where baseline accepted paths are substantially shorter, as well as to both target sizes.

\begin{figure}[t]
 \centering
 \includegraphics[width=\linewidth]{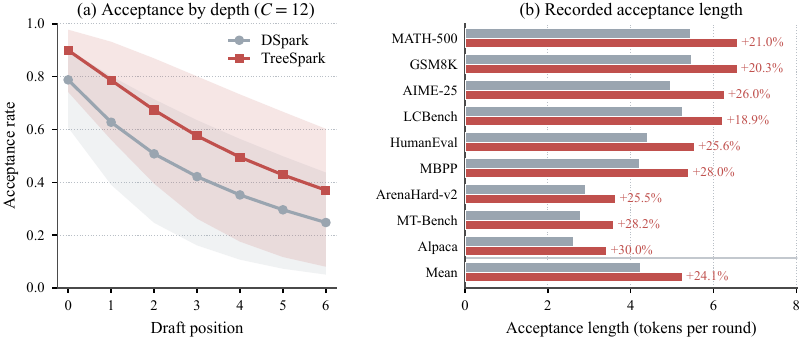}
 \caption{\textbf{Acceptance on Qwen3-4B.} Left: position-wise acceptance at $C=12$, averaged over nine benchmarks; shading spans their minimum and maximum. Right: recorded acceptance lengths from Table~\ref{tab:acceptance-length}, including the target bonus when emitted. The panels use separate evaluation records.}
 \label{fig:acceptance}
\end{figure}

Figure~\ref{fig:acceptance} examines how these gains extend through a proposal. At $C=12$, mean acceptance at the first draft position increases from 78.8\% to 90.0\%, while acceptance at the seventh increases from 24.7\% to 37.0\%. The improvement persists at every plotted depth. Since accepting a later token requires a matching prefix, this pattern is consistent with improved retention of useful continuations. These measurements evaluate the complete system; Figure~\ref{fig:summary} isolates the additional effect of supervision at fixed capacity and candidate policy. The right panel reports the same benchmark acceptance lengths as Table~\ref{tab:acceptance-length}, whose means are 4.23 and 5.25.

\subsection{Serving Efficiency and Training Alignment}
\label{sec:throughput-results}

The acceptance gains translate into higher throughput on every single-request benchmark (Table~\ref{tab:throughput-c1}). Relative to DSpark, average gains on the general suite are 15.1\% for Qwen3-4B and 17.9\% for Qwen3-8B; on writing they reach 21.9\% and 25.0\%. Thus, the additional accepted tokens compensate for the cost of constructing and verifying a branching proposal in these settings.

\begin{table}[t]
\caption{System throughput on Qwen3-4B (tokens/s) at total concurrency $C$. Serving configurations are specified in Section~\ref{sec:setup}. $\Delta$ is the gain over DSpark within each setting; its Mean entry averages per-benchmark gains.}
\label{tab:concurrency}
\centering
\footnotesize
\setlength{\tabcolsep}{2pt}
\begin{tabular*}{\textwidth}{@{\extracolsep{\fill}}lcccccccccc@{}}
\toprule
\multicolumn{11}{@{}l}{\emph{System throughput (tokens/s)}} \\
\midrule
 & GSM8K & \hdr{MATH}{500} & \hdr{AIME}{25} & \hdr{Human}{Eval} & MBPP & \hdr{LiveCode}{Bench} & \hdr{MT-}{Bench} & Alpaca & \hdr{Arena}{Hard-v2} & Mean \\
\midrule
\multicolumn{11}{@{}l}{\emph{$C=6$}} \\
Vanilla & \phantom{0}68.79 & \phantom{0}58.50 & \phantom{0}45.27 & \phantom{0}59.66 & \phantom{0}62.80 & \phantom{0}49.00 & \phantom{0}56.86 & \phantom{0}59.37 & \phantom{0}54.71 & \phantom{0}57.22 \\
\dspark{} & 235.51 & 223.87 & 183.69 & 173.60 & 167.44 & 181.85 & \phantom{0}92.84 & 106.23 & 118.01 & 164.78 \\
\ours{} & \textbf{256.65} & \textbf{256.29} & \textbf{203.90} & \textbf{201.66} & \textbf{179.66} & \textbf{208.52} & \textbf{127.50} & \textbf{124.15} & \textbf{141.94} & \textbf{188.92} \\
$\Delta$ (\%) & +9.0 & +14.5 & +11.0 & +16.2 & +7.3 & +14.7 & +37.3 & +16.9 & +20.3 & +16.3 \\
\midrule
\multicolumn{11}{@{}l}{\emph{$C=12$}} \\
Vanilla & 124.44 & 103.84 & \phantom{0}68.84 & \phantom{0}96.04 & 102.51 & \phantom{0}87.29 & \phantom{0}72.46 & \phantom{0}95.62 & 102.11 & \phantom{0}94.79 \\
\dspark{} & 341.53 & 374.50 & 274.46 & 258.41 & 250.87 & 284.77 & 139.06 & 173.52 & 195.02 & 254.68 \\
\ours{} & \textbf{381.60} & \textbf{435.55} & \textbf{314.90} & \textbf{308.42} & \textbf{281.77} & \textbf{316.64} & \textbf{186.24} & \textbf{209.82} & \textbf{223.15} & \textbf{295.34} \\
$\Delta$ (\%) & +11.7 & +16.3 & +14.7 & +19.4 & +12.3 & +11.2 & +33.9 & +20.9 & +14.4 & +17.2 \\
\midrule
\multicolumn{11}{@{}l}{\emph{$C=36$}} \\
Vanilla & \phantom{0}316.38 & \phantom{0}266.40 & 133.88 & 209.56 & 189.49 & 196.07 & 125.51 & 214.07 & 257.57 & 212.10 \\
\dspark{} & \phantom{0}948.63 & \phantom{0}944.46 & 431.36 & 527.65 & 550.19 & 641.40 & 239.36 & 413.14 & 473.90 & 574.45 \\
\ours{} & \textbf{1021.66} & \textbf{1086.23} & \textbf{457.30} & \textbf{652.25} & \textbf{554.28} & \textbf{747.13} & \textbf{258.87} & \textbf{521.53} & \textbf{518.70} & \textbf{646.44} \\
$\Delta$ (\%) & +7.7 & +15.0 & +6.0 & +23.6 & +0.7 & +16.5 & +8.2 & +26.2 & +9.5 & +12.6 \\
\bottomrule
\end{tabular*}
\end{table}

Under concurrent serving, \method\ improves system throughput on all nine benchmarks at each tested concurrency (Table~\ref{tab:concurrency}). Mean per-benchmark gains are 16.3\% at $C=6$, 17.2\% at $C=12$, and 12.6\% at $C=36$, each relative to DSpark within the corresponding serving configuration in Section~\ref{sec:setup}. Per-request decoding throughput also improves, by 15.5\%, 17.6\%, and 11.8\%, respectively (Appendix~\ref{app:request-results}). These results show that the benefit remains measurable when requests share accelerator execution.

\begin{figure}[t]
 \centering
 \includegraphics[width=\linewidth]{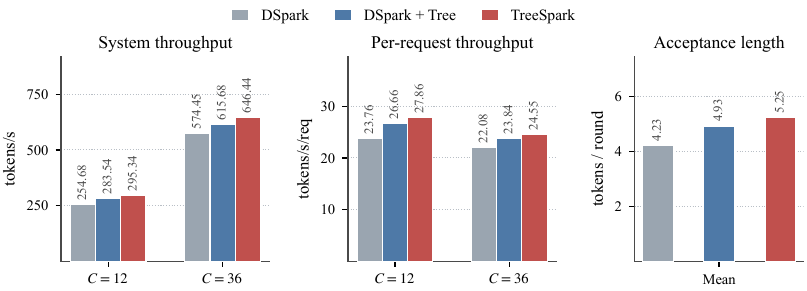}
 \caption{\textbf{Effect of training alignment on Qwen3-4B.} DSpark + Tree matches the capacity and inference policy of \method\ but omits both tree losses. All panels report means over nine benchmarks. DSpark and \method\ acceptance means match Table~\ref{tab:acceptance-length}.}
 \label{fig:summary}
\end{figure}

To separate the contribution of training from the inference extension, Figure~\ref{fig:summary} compares \method\ with DSpark + Tree. This control uses identical training and inference settings, including head dimension 192, except that both tree losses are disabled. Mean acceptance length increases from 4.93 to 5.25, a gain of approximately 6.4\%. Mean system throughput rises from 283.54 to 295.34 tokens/s at $C=12$ and from 615.68 to 646.44 at $C=36$, gains of 4.2\% and 5.0\%, respectively. Mean per-request throughput also increases at both concurrency levels. The improvement under a fixed candidate policy supports aligning draft learning with the inference extension. System throughput exceeds the control on six of nine benchmarks at each concurrency; Appendix~\ref{app:detailed-figures} provides the full breakdown.

\subsection{Contributions of Supervision and Capacity}
\label{sec:ablation}

\begin{table}[t]
\caption{Component ablations on Qwen3-4B at $C=12$ (system tokens/s). Loss ablations remove only the indicated objective; the capacity ablation uses head dimension 128 instead of 192. $\Delta$ compares each variant with DSpark; its Mean entry averages benchmark gains. Bold marks column maxima.}
\label{tab:ablation-c12}
\centering
\footnotesize
\setlength{\tabcolsep}{2pt}
\begin{tabular*}{\textwidth}{@{\extracolsep{\fill}}lcccccccccc@{}}
\toprule
\multicolumn{11}{@{}l}{\emph{System throughput (tokens/s)}} \\
\midrule
 & GSM8K & \hdr{MATH}{500} & \hdr{AIME}{25} & \hdr{Human}{Eval} & MBPP & \hdr{LiveCode}{Bench} & \hdr{MT-}{Bench} & Alpaca & \hdr{Arena}{Hard-v2} & Mean \\
\midrule
\dspark{} & 341.53 & 374.50 & 274.46 & 258.41 & 250.87 & 284.77 & 139.06 & 173.52 & 195.02 & 254.68 \\
\midrule
\ours{} & \textbf{381.60} & \textbf{435.55} & 314.90 & \textbf{308.42} & 281.77 & 316.64 & \textbf{186.24} & \textbf{209.82} & 223.15 & \textbf{295.34} \\
$\Delta$ (\%) & +11.7 & +16.3 & +14.7 & +19.4 & +12.3 & +11.2 & +33.9 & +20.9 & +14.4 & +17.2 \\
\midrule
w/o $\cL_{\mathrm{surv}}$ & 364.20 & 412.68 & 297.14 & 297.21 & 281.76 & 292.01 & 171.53 & 200.31 & 220.20 & 281.89 \\
$\Delta$ (\%) & \phantom{0}+6.6 & +10.2 & \phantom{0}+8.3 & +15.0 & +12.3 & \phantom{0}+2.5 & +23.3 & +15.4 & +12.9 & +11.9 \\
\midrule
w/o $\cL_{\mathrm{cov}}$ & 373.38 & 432.72 & 317.94 & 300.38 & \textbf{290.83} & \textbf{321.24} & 172.76 & 207.75 & 223.21 & 293.36 \\
$\Delta$ (\%) & \phantom{0}+9.3 & +15.5 & +15.8 & +16.2 & +15.9 & +12.8 & +24.2 & +19.7 & +14.5 & +16.0 \\
\midrule
Head dim.\ 128 & 369.73 & 424.70 & \textbf{319.65} & 301.80 & 289.33 & 302.82 & 171.09 & 203.91 & \textbf{223.69} & 289.64 \\
$\Delta$ (\%) & \phantom{0}+8.3 & +13.4 & +16.5 & +16.8 & +15.3 & \phantom{0}+6.3 & +23.0 & +17.5 & +14.7 & +14.6 \\
\bottomrule
\end{tabular*}
\end{table}

Table~\ref{tab:ablation-c12} separates the two tree objectives and attention capacity at $C=12$. Each loss ablation retains head dimension 192 and removes only the indicated term. Without survival supervision, the mean gain over DSpark falls from 17.2\% to 11.9\%; without coverage, it falls to 16.0\%. Survival contributes the larger aggregate effect, while coverage adds a smaller improvement consistent with its role near local selection boundaries. Individual benchmarks sometimes favor an ablation, so the evidence supports complementarity in aggregate rather than uniform gains from each term.

With both objectives and the tree policy fixed, increasing head dimension from 128 to 192 raises mean system throughput from 289.64 to 295.34 tokens/s, a 2.0\% improvement. Mean per-request throughput increases by 2.1\% (Appendix~\ref{app:request-results}). These results support allocating additional attention capacity to tree-supervised drafting. The 128-dimensional variant still improves over DSpark by 14.6\%, while the capacity-matched comparison in Section~\ref{sec:throughput-results} establishes an additional benefit from tree supervision itself. The full configuration achieves the highest mean throughput in both ablation tables.

\section{Conclusion}
\label{sec:conclusion}

\method\ demonstrates how a deployment policy can guide the training of a parallel drafter. Training on model-generated parent states under the deployed selection policy directs reference supervision toward the continuations available for target verification. The controlled comparison yields longer mean accepted paths and higher mean throughput at fixed attention capacity and inference policy. This evidence supports treating candidate admission and prefix retention as integral to draft learning, alongside reference prediction accuracy. The resulting extension of DFlash and DSpark preserves shared backbone computation while improving how its predictions are used. These findings motivate draft objectives that account for the states and selection constraints encountered during candidate generation and verification.

\bibliography{iclr2027_conference}
\bibliographystyle{iclr2027_conference}

% \clearpage
\appendix
\section{Tree Construction and Greedy Verification}
\label{app:verification}

This appendix specifies the candidate policy used to extend the semi-autoregressive drafter in Section~\ref{sec:tree}. The construction retains alternative continuations from shared backbone outputs. The verification arguments state the conditions under which this deployed procedure preserves greedy target output.

\subsection{Candidate Expansion and Global Packing}
\label{app:construction}

The root $\rho$ represents the final committed token and has depth zero. For a non-root node $v$, let $\pa(v)$ be its parent, $t_v$ its token, and $d(v)$ its depth. Initialize the frontier as $\cF_0=\{\rho\}$, the candidate pool as $\cC=\varnothing$, and the root score as $s(\rho)=0$. At depth $d$, every $u\in\cF_{d-1}$ produces the $k$ highest-probability tokens under $q_{\theta,d}(\cdot\mid u,x)$. Each resulting child receives the joint log score
\begin{equation}
 s(v)=s(\pa(v))+
       \log q_{\theta,d(v)}(t_v\mid\pa(v),x).
 \label{eq:score}
\end{equation}
All these children are added to $\cC$. Among the children created at this depth, the $k$ highest-scoring nodes form $\cF_d$. Nodes outside the next frontier remain in the pool and remain eligible for verification. A path can therefore terminate before depth $B$ without losing its previously generated nodes.

The first depth contains at most $k$ candidates and each subsequent depth contains at most $k^2$. Consequently, $|\cC|\leq k+(B-1)k^2$. Global packing selects up to $N$ non-root nodes from this pool in decreasing score order, breaking score ties by increasing depth. The selected nodes and root define $\cT$. Their physical arrangement in the verification input preserves parent-before-child order. Distinct paths that end in the same token remain distinct nodes because the target conditionals depend on their complete prefixes.

\begin{proposition}[Prefix closure]
\label{prop:closure}
Suppose $\cC$ contains every non-root ancestor of each of its nodes. Any initial segment of the score ordering above, together with $\rho$, is prefix closed.
\end{proposition}
\begin{proof}
A conditional log probability is nonpositive, so Equation~\ref{eq:score} implies $s(v)\leq s(\pa(v))$. A strictly higher parent score places the parent before its child. Equal scores also place the parent first because its depth is smaller. Applying this argument recursively shows that every ancestor precedes the node. Thus, selecting an initial segment cannot retain a node without its ancestors.
\end{proof}

Prefix closure makes the global budget operationally useful: every selected node has a complete path available for target evaluation, and no extra ancestor positions need to be inserted after packing. The result depends on the ordering rule, not on the accuracy of draft probabilities. Appendix~\ref{app:numerics} describes finite precision differences between the training and serving selection operators.

\subsection{Attention Contexts and Verification}
\label{app:greedy}

Write the committed prefix as $x=x_{<t}x_t$, where $x_t$ is the anchor. The target retains the cache for $x_{<t}$ and processes the anchor together with the packed non-root nodes. The attention mask for node $u$ exposes $x_{<t}$ and the root-to-$u$ path, including $u$ itself. Siblings and unrelated branches are masked. A node at depth $d$ uses absolute position $t+d$, independently of its packed index. These choices give each node the causal context of its corresponding sequential continuation.

Verification begins at the root. The output logits at the current node predict the next token. If the target argmax occurs among its children, that token is accepted and verification continues at the matching child. Otherwise, the target argmax is emitted as a bonus token and the round ends. A leaf likewise supplies the next target token as the bonus. End-of-sequence and output-length constraints can terminate the process earlier.

\begin{proposition}[Greedy output preservation]
\label{prop:greedy}
Assume that tree-masked and sequential target evaluations agree on identical token paths and use the same argmax tie rule. For any prefix-closed candidate tree, the verifier emits the same tokens as autoregressive greedy decoding with $P$.
\end{proposition}
\begin{proof}
At the root, both procedures evaluate the same committed prefix and select the same next token. If the tree contains that token as a child, accepting it extends the prefix identically. The child attention context then agrees with sequential decoding on the extended prefix. Induction gives agreement at every accepted depth. At the first missing child, or at a leaf, the bonus token is selected from target logits on the same prefix and therefore also agrees. Applying the argument after every round establishes equality of the generated sequence.
\end{proof}

Target evaluation determines every committed token, while draft learning affects how many can be committed together. The guarantee therefore depends on per-path target evaluations and argmax tie rules, independently of the supervision used to train the drafter. In finite precision, different attention kernels can affect nearly tied logits, so the stated numerical assumption remains relevant. The evaluated procedure uses greedy decoding.

\subsection{Acceptance Length Convention}

Let $A_{\mathrm g}$ denote the number of accepted non-root draft nodes. An ordinary completed round commits $A_{\mathrm g}+1$ new tokens, including the target bonus. The implementation reports this committed length as acceptance length. A round stopped by an accepted end token or an output limit can omit the bonus. The anchor is already committed and is excluded. Training diagnostics at depth $d$ describe draft-path survival and should therefore be distinguished from a committed-length metric that includes the bonus.

\section{Training Objectives and Replay Details}
\label{app:training}

This appendix gives the full reductions behind Section~\ref{sec:objectives} and their integration with the inherited reference losses. It specifies how replay exposes the parallel drafter to proposed parent states and how the two additional objectives allocate supervision across its candidates.

\subsection{Reference Supervision and Replay}
\label{app:target_proxy}

A training sequence supplies reference tokens and cached target features. Aligned reference logits are obtained with the frozen target output projection and reused by both ordinary distillation and the tree objectives. The draft backbone input embeddings and vocabulary output projection are initialized from the target and remain frozen. The Markov embedding $e_\theta$ and projection $W_\theta$ in Equation~\ref{eq:draft} are separate trainable parameters. They are optimized jointly with the draft transformer layers; the target model remains frozen.

Index replayed anchors by $a$. The frozen reference distribution at depth $d$ is $\bar p_{a,d}=P(\cdot\mid x_a,y_{a,<d})$. For every proposed node, define
\begin{equation}
 m_a(\rho)=1,\qquad
 m_a(v)=m_a(\pa(v))\,
        \bar p_{a,d(v)}(t_v).
 \label{eq:mass}
\end{equation}
This recursively assigns a reference-derived mass to a complete proposed path. At the first depth, the teacher context is the committed reference anchor context. At later depths, the target distribution is shared across parents even when their paths differ from the reference. Exact target conditionals on those alternative paths would require additional target evaluations.

The replay proceeds as follows.
\begin{enumerate}
 \item Compute the parallel draft backbone outputs and baseline reference losses. Uniformly subsample up to $A_{\mathrm{tree}}$ valid anchors without replacement for tree supervision.
 \item Initialize one frontier root per selected anchor, with zero joint log score and unit reference mass.
 \item At each depth, evaluate the Markov head on the frontier parents. Record differentiable log probabilities of proposed children and coverage penalties for target top-$J$ tokens.
 \item Detach candidate identities and joint scores, add the proposed children to the pool, propagate detached reference masses, and select the next frontier.
 \item Apply the global node budget. Accumulate survival loss on valid packed nodes and coverage loss on valid expanded parents, then combine them with the baseline objective.
\end{enumerate}

The next update repeats the construction with updated parameters. The topology is therefore responsive to learning even though each gradient is evaluated with the current selection held fixed. Replay exposes the correction head to its own proposed parents within a block; the preceding anchor context and teacher distributions continue to come from reference data. The method uses this separation to obtain branch-dependent draft training without additional target tree evaluations.

All \method\ results use the first-order Markov correction in Equation~\ref{eq:draft}. Within a fixed anchor context, paths ending in the same token at the same depth share a correction. They remain separate candidates because target verification conditions on their full prefixes.

\subsection{Packed-Edge Survival Objective}

Let $e_{a,d}\in\{0,1\}$ mark a valid supervised depth and let $I_a(v)$ indicate that node $v$ is selected by global packing. Define $q_{a,v}=q_{\theta,d(v)}(t_v\mid\pa(v),x_a)$. With $\sg$ denoting stop-gradient, Equation~\ref{eq:survival} expands to
\begin{equation}
 \cL_{\mathrm{surv}}=
 -\frac{\displaystyle\sum_a\sum_{v\in\cC_a}
       e_{a,d(v)}I_a(v)\,\sg[m_a(v)]\,\log q_{a,v}}
       {\displaystyle\sum_a\sum_{v\in\cC_a}
       e_{a,d(v)}I_a(v)\,\sg[m_a(v)]+\epsilon}.
 \label{eq:survival_full}
\end{equation}
There is no additional explicit depth decay in this term. Its preference over paths arises from their reference masses and the packed-node mask. The denominator is detached along with those quantities. Conditional log probabilities remain differentiable, including their vocabulary normalization.

The path weights connect local prediction updates to the prefix dependence of acceptance. The packing mask further concentrates this supervision on predictions admitted by the deployed policy. Together, they adapt reference likelihood learning to the continuations that the inference extension makes available to the target.

\subsection{Frontier Coverage Objective}

For each anchor and depth, let $J_{a,d}$ contain the target top-$J$ tokens and write $H_{a,d}=\sum_{i\in J_{a,d}}\bar p_{a,d}(i)$. Let $h_{a,d}(u)$ be the parent penalty in Equation~\ref{eq:hinge}. The complete coverage objective is
\begin{equation}
 \cL_{\mathrm{cov}}=
 \frac{\displaystyle\sum_{a,d}\eta_{a,d}
             \sum_{u\in\cF_{a,d-1}}h_{a,d}(u)}
      {\displaystyle\sum_{a,d}\eta_{a,d}
             |\cF_{a,d-1}|H_{a,d}+\epsilon},
 \qquad
 \eta_{a,d}=e_{a,d}\exp\!\left(-\frac{d-1}{\gamma}\right).
 \label{eq:coverage}
\end{equation}
The decay scale $\gamma>0$ gives earlier candidate decisions greater weight. Every expanded parent contributes, including parents whose children are subsequently excluded by packing. This term uses neither a packing mask nor a parent path-mass multiplier. Its weighting is the product of target token probability and depth weight, with the same aggregate weight in the denominator.

The target distributions, target top-$J$ identities, and normalization are fixed. Gradients pass through the selected draft logits and the $k$th-largest boundary value. Discrete changes in the identity of that boundary are handled by reconstructing the replay at later updates.

\subsection{Baseline Losses and Distributed Normalization}
\label{app:baseline}

The baseline objective in Equation~\ref{eq:total} is
\begin{equation}
 \cL_{\mathrm{base}}=
 \lambda_{\mathrm{CE}}\cL_{\mathrm{CE}}
 +\lambda_1\cL_1
 +\lambda_{\mathrm{conf}}\cL_{\mathrm{conf}}.
 \label{eq:base}
\end{equation}
Cross-entropy supervises the reference continuation token. The distribution loss uses the unhalved $L_1$ distance between the target distribution and the draft distribution evaluated with the reference parent. The confidence head predicts the detached overlap score
\begin{equation}
 \alpha_{a,d}
 =1-\frac12\left\|\bar p_{a,d}
                  -q^{\mathrm{ref}}_{\theta,a,d}\right\|_1,
 \label{eq:confidence_target}
\end{equation}
using binary cross-entropy. Each baseline term is reduced with a validity mask and exponential depth decay, divided by the corresponding summed valid weight.

For ordinary single-proposal speculative rejection sampling, the overlap in Equation~\ref{eq:confidence_target} equals the one-step acceptance probability \citep{leviathan2023fast,chen2023speculative}. This explains its use in the inherited confidence objective. Greedy tree acceptance instead depends on whether the target argmax path is present, so this overlap is not substituted for accepted tree depth.

For each loss component, the implementation sums valid weights across distributed workers to form its denominator. Each local numerator is divided by that global denominator and multiplied by the number of workers to compensate for gradient averaging. The resulting gradient corresponds to the globally normalized objective, rather than an equal average of worker-specific losses with different valid weights. A stabilizer $\epsilon=10^{-6}$ protects the reductions, and components with no valid supervision are disabled. Reference distillation remains active when tree supervision is disabled.

\section{Analysis of the Supervision Objectives}
\label{app:analysis}

The following identities and gradient calculations explain the weighting and complementary roles of the two objectives. They support the training design by relating retained prefixes to candidate coverage and characterizing the signals applied to proposed parent states.

\subsection{Path Mass as a Supervision Weight}

We use the standard prefix-coverage identity underlying draft-tree selection \citep{ringel2026ddtree} to motivate the path weights in our training objective. The derivation below states the identity under exact target conditionals before relating it to the reference distributions used in replay.

Fix a committed prefix $x$ and a prefix-closed tree $\cT$. For a node $v$ at depth $d$ with token path $v_{1:d}$, its exact target path probability is
\begin{equation}
 m_P(v)=\prod_{j=1}^{d}P(v_j\mid x,v_{<j}).
 \label{eq:exact_mass}
\end{equation}
Let $Y$ be sampled autoregressively from $P$, and let $A$ be the length of its longest prefix contained in $\cT$, truncated at $B$.

\begin{proposition}[Depthwise coverage identity]
\label{prop:coverage}
The expected contained-prefix length is
\begin{equation}
 \E_P[A\mid x,\cT]
 =\sum_{d=1}^{B}S_d^P,
 \qquad
 S_d^P=\sum_{\substack{v\in\cT\\d(v)=d}}m_P(v).
 \label{eq:coverage_identity}
\end{equation}
\end{proposition}
\begin{proof}
For a fixed depth $d$, distinct nodes represent disjoint token prefixes. The event $A\geq d$ occurs precisely when $Y_{1:d}$ matches one of those prefixes. Its probability is therefore the sum of their path probabilities, $S_d^P$. Applying the tail-sum identity to $A\in\{0,\ldots,B\}$ gives the result.
\end{proof}

The identity motivates weighting deeper predictions through their complete prefixes. Survival uses this principle as detached likelihood weights during replay. Equation~\ref{eq:coverage_identity} characterizes a target-sampled path in a fixed candidate tree; the implemented objective is a supervised surrogate, and greedy acceptance is measured with the deployed verifier.

Reference-derived masses have an exact coverage interpretation under the auxiliary product distribution
\begin{equation}
 \bar P(Y_{1:B}\mid x)=\prod_{d=1}^{B}\bar p_d(Y_d).
 \label{eq:reference_product}
\end{equation}
Thus, summing $m(v)$ over depth-$d$ nodes measures coverage under $\bar P$. It approximates coverage under the autoregressive target to the extent that reference distributions remain representative on alternative paths. For either distribution, coverage is nonincreasing with depth in a prefix-closed tree because extending a contained path imposes an additional constraint.

\subsection{Survival Gradients on Retained Edges}
\label{app:gradient}

Consider a fixed replay and one valid parent $u$. Let $K_u$ be the tokens of its packed children, let $a_i$ denote their detached reference path masses, and let $M_u=\sum_{i\in K_u}a_i$. Before global normalization, the parent contributes
\begin{equation}
 L_u=-\sum_{i\in K_u}a_i\log q_i,
 \qquad
 \frac{\partial L_u}{\partial z_j}
 =M_uq_j-\ind\{j\in K_u\}a_j.
 \label{eq:survival_gradient}
\end{equation}
For $M_u>0$, this is $M_u$ times cross-entropy against a distribution assigning $a_i/M_u$ to retained children and zero to other tokens. The parent path mass cancels from this local normalized target but scales the contribution of the parent relative to other states. Shared draft parameters couple these local updates across parents and depths.

Positive likelihood targets are assigned to retained edges. Omitted tokens receive only normalization gradients from this term, which explains the need for coverage and reference distillation. The derivative is conditional on the current replay; subsequent updates reconstruct parent states and candidate choices. Acceptance and throughput evaluation measure the resulting deployment behavior.

\subsection{Coverage Gradients and Saturation}
\label{app:hinge}

For one supervised token, write $r_i=\delta-z_i+b$ and $\phi(r)=\min\{c,\max(0,r)\}$. Away from the two clipping boundaries, its derivative with respect to $r$ is one when $0<r<c$ and zero otherwise. On a region with fixed top-$k$ identities,
\begin{equation}
 \nabla_z\phi(r_i)=
 \begin{cases}
  \nabla_z b-\nabla_z z_i,
       & \delta-c<z_i-b<\delta,\\
  0,   & \text{otherwise}.
 \end{cases}
 \label{eq:hinge_gradient}
\end{equation}
Consequently, a target-important token near the selection boundary is promoted relative to that boundary. A token sufficiently above the margin needs no correction. A token far below the boundary contributes a capped penalty without a direct hinge gradient, preventing large logit gaps from dominating this auxiliary objective. Such tokens remain supervised by the baseline losses and can enter the active interval as shared parameters change.

The boundary value is differentiable. If the supervised token itself defines the $k$th logit, the derivatives of $z_i$ and $b$ cancel and its penalty equals $\min\{c,\delta\}$. Hence the loss need not vanish when the target top-$J$ set is already contained in the draft top-$k$ set. It is a local margin surrogate for useful candidate coverage, not an exact set-membership loss.

Every scalar penalty lies in $[0,c]$, and Equation~\ref{eq:coverage} normalizes by the same total target and depth weights that multiply those penalties. It follows that $0\leq\cL_{\mathrm{cov}}\leq c$. This bound gives the auxiliary coefficient $\lambda_c$ a consistent interpretation across changes in the number of replayed anchors or expanded parents. The objective places no pairwise ordering constraints on target top-$J$ tokens, while their target probabilities determine the relative importance of active corrections.

\subsection{Diagnostics on the Packed Tree}
\label{app:diagnostics}

We record two quantities using the packed-node mask. Reference mass coverage at depth $d$ is
\begin{equation}
 \bar S_d=\sum_{\substack{v\in\cT\\d(v)=d}}m(v).
 \label{eq:mass_diagnostic}
\end{equation}
It measures how much probability the selected tree covers under Equation~\ref{eq:reference_product}. Reference argmax survival uses $\bar g_d=\arg\max_i\bar p_d(i)$ and records
\begin{equation}
 \bar G_d=\ind\{\bar g_{1:d}\text{ is a path in }\cT\}.
 \label{eq:argmax_diagnostic}
\end{equation}
The sum of these indicators is the contained depth of the reference argmax sequence.

For comparison, actual greedy accepted depth replaces $\bar g_d$ with the recursively conditioned token $g_d=\arg\max_i P(i\mid x,g_{<d})$. These sequences coincide when their conditioning contexts agree. The diagnostics therefore characterize the selected training tree under its available supervision, while deployment evaluation measures actual target acceptance. Batch averages use valid anchors at each depth, so their denominators can differ across depths.

\section{Implementation and Computational Cost}
\label{app:implementation}

\subsection{Configuration and Capacity}

The reference configuration uses $B=7$, $k=4$, and $N=32$ non-root nodes. A full seven-depth expansion produces $4+6\cdot16=100$ candidates before packing. A full target input contains the root anchor and 32 selected nodes, for 33 positions. The maximum ordinary committed length is eight tokens, comprising seven accepted draft tokens and one target bonus. This budget convention counts the anchor separately.

The correction rank is $r=256$. Coverage uses $J=4$, margin $\delta=1$, cap $c=4$, and depth-decay scale $\gamma=4$. Baseline coefficients are $\lambda_{\mathrm{CE}}=0.1$, $\lambda_1=0.9$, and $\lambda_{\mathrm{conf}}=1$; the two tree coefficients are $\lambda_s=\lambda_c=0.1$. Up to 128 tree anchors are selected from 512 backbone anchor blocks per sequence. Training records its tree hyperparameters, and inference receives the corresponding branching and packing arguments explicitly. Deployment with the intended training policy requires matching those values.

The complete systems for Qwen3-4B and Qwen3-8B use draft attention head dimension 192, compared with 128 in the inherited configuration. The motivation is to provide a larger attention representation for the shared contextual prediction that must support multiple parent-conditioned distributions during tree replay. At fixed attention head counts, widening each head increases the dimensions of the query, key, and value projections and the input to the output projection. The residual hidden size, layer count, Markov rank $r=256$, and tree budget are unchanged. The capacity increase is therefore localized to the draft attention computation and is used during both training and inference.

Attention normalization, rotary embeddings, and scaling follow the selected head dimension. The affected projection sizes and draft KV storage scale by $192/128=1.5$; this factor does not describe the parameter count or memory footprint of the entire model. DSpark + Tree and each single-loss ablation retain dimension 192. The 128-dimensional variant retains both objectives and the candidate policy, allowing its throughput comparison to assess the capacity choice separately from the supervision.

\subsection{Replay Cost and Activation Memory}

The backbone outputs are shared across all parents at an anchor and depth. Following one backbone evaluation, replay computes at most $k$ Markov corrections per anchor and depth. Their rank-$r$ vocabulary projection costs $O(kr|\mathcal V|)$; it does not repeat transformer computation for each branch. Tree metadata requires $O(A_{\mathrm{tree}}Bk^2)$ entries per sequence, while vocabulary logits are the dominant intermediate tensors.

Anchor subsampling limits the number of tree states used in an update. Checkpointing reconstructs each depth of parent-conditioned logits during backward, reducing the need to retain vocabulary-sized activations. Target distributions are processed by depth. The detached pool stores token identities, parent pointers, path scores, reference masses, and packing masks.

These choices preserve the main source of parallel drafting efficiency: a shared backbone supplies the block, and a lightweight head supplies branch dependence. Tree supervision still adds draft computation and backward work. Its specific target-side advantage is that it reuses the teacher information already needed by reference distillation.

\subsection{Batched Construction and Verification}

The first frontier has one parent per request and later frontiers have up to $k$. Compatible frontier shapes allow active requests to share each depth-wise correction and selection call. There remain $B$ sequential head stages, but the implementation avoids a separate host expansion loop for every request.

Local top-$k$ identities can be selected from raw logits. Joint path comparison additionally requires the row normalization:
\begin{equation}
 \log q_i=z_i-\operatorname{logsumexp}_{j\in\mathcal V}z_j.
 \label{eq:normalizer}
\end{equation}
Computing this value for selected tokens avoids materializing a normalized vocabulary tensor solely to retain a few entries. The normalizer itself is essential because candidates from different parents compete by joint probability.

Parent pointers are converted on device into ancestor masks and path retrieval indices. The target verifies the packed trees in a batch, and accepted paths are retrieved per request. Cache management retains valid states associated with committed paths and invalidates rejected positions using per-request masks. Compaction is deferred until needed, allowing requests to commit different numbers of tokens without imposing a common accepted depth.

\subsection{Numerical Selection Conventions}
\label{app:numerics}

Training and serving use the same expansion, frontier pruning, and node-budget rules. The serving reference additionally prefers smaller depth on tied joint scores, as specified in Appendix~\ref{app:construction}. Training uses floating point top-$N$ without this depth preference. The packed sets therefore coincide for matching scores without boundary ties; exact agreement at ties requires a common tie rule. Proposition~\ref{prop:closure} assumes the stated depth-aware ordering.

The optimized serving path performs local selection, subtraction, and exponentiation in the input precision, with float32 accumulation for the normalizer. Scores can consequently differ from full float32 log-softmax near a selection boundary. These operator-level differences concern the selected candidates. Greedy output preservation follows from the target verification conditions in Proposition~\ref{prop:greedy}.

\section{Evaluation Protocol and Supplementary Results}
\label{app:experiments}

\subsection{Data and Training Protocol}
\label{app:experimental-protocol}

\paragraph{Training mixture.}
All main draft comparisons use a common mixture of 50,000 prompts. The mathematics sources are GSM8K \citep{cobbe2021gsm8k}, MATH \citep{hendrycks2021math}, and AIME 2024. Code prompts come from the sanitized MBPP training split \citep{austin2021mbpp}, the Python subset of Magicoder OSS-Instruct \citep{wei2024magicoder}, APPS \citep{hendrycks2021apps}, and CodeContests \citep{li2022alphacode}. Dialogue and instruction prompts come from OASST1 \citep{kopf2023openassistant}, Dolly-15K \citep{databricks2023dolly}, and WildChat \citep{zhao2024wildchat}. Writing prompts come from the Writer-R1 training collection \citep{zhao2026writerr1}. Writer-R1 and WritingBench are distinct sources; the latter supplies writing evaluation prompts.

The mixture retains the first user turn of each dialogue. For instruction records, a nonempty input field is appended to the instruction. Prompts longer than 2,048 tokens are removed before sampling. A pseudorandom generator with seed 42 samples each source without replacement, shuffles the combined pool of 50,123 records, and retains the first 50,000. Table~\ref{tab:training-mixture} reports the resulting counts after filtering and truncation. These differ slightly from the sampling quotas. The final composition is 29.92\% mathematics, 40.15\% code, 19.94\% dialogue and instructions, and 9.99\% writing. The drafter learns from cached target continuations and features on these prompts, with a maximum training sequence length of 4,096 tokens.

\begin{table}[t]
 \centering
 \caption{Composition of the 50,000-prompt training mixture after length filtering, sampling, shuffling, and truncation. Counts refer to records, not unique prompt strings.}
 \label{tab:training-mixture}
 \small
 \begin{tabular*}{\linewidth}{@{\extracolsep{\fill}}llrr@{}}
 \toprule
 Category & Source & Records & Share (\%) \\
 \midrule
 Mathematics & GSM8K & 7,450 & 14.90 \\
 & MATH & 7,480 & 14.96 \\
 & AIME 2024 & 30 & 0.06 \\
 Code & MBPP & 120 & 0.24 \\
 & Magicoder OSS-Instruct, Python & 9,481 & 18.96 \\
 & APPS & 3,989 & 7.98 \\
 & CodeContests & 6,485 & 12.97 \\
 Dialogue and instructions & OASST1 & 3,488 & 6.98 \\
 & Dolly-15K & 2,990 & 5.98 \\
 & WildChat & 3,493 & 6.99 \\
 Writing & Writer-R1 & 4,994 & 9.99 \\
 \midrule
 Total & & 50,000 & 100.00 \\
 \bottomrule
 \end{tabular*}
\end{table}

\paragraph{Evaluation workloads.}
The general suite spans mathematics with GSM8K, MATH-500 \citep{hendrycks2021math,lightman2023verify}, and AIME 2025; code with HumanEval \citep{chen2021humaneval}, MBPP, and LiveCodeBench \citep{jain2025livecodebench}; and dialogue or instruction following with MT-Bench \citep{zheng2023mtbench}, Alpaca \citep{taori2023alpaca}, and ArenaHard-v2 \citep{lmarena2025arenahard}. The evaluation year for AIME is 2025, distinct from the 2024 problems used in training. These datasets supply prompts for measuring decoding efficiency; the reported values are throughput and acceptance, rather than task accuracy or writing quality scores.

The general suite uses 500 GSM8K prompts, 500 MATH-500 problems, 30 AIME 2025 problems, 164 HumanEval prompts, 256 MBPP prompts, 500 LiveCodeBench prompts, 80 MT-Bench prompts, 500 Alpaca prompts, and 500 ArenaHard-v2 prompts. Each request uses the first user turn of its record. Where a local prompt pool exceeds the requested count, the evaluator shuffles it with a fixed pseudorandom seed before selecting the subset. All compared methods use the same selected prompts.

WritingBench \citep{wu2025writingbench} covers Academic and Engineering, Advertising and Marketing, Education, Finance and Business, Literature and Arts, and Politics and Law. We evaluate all 1,000 prompts, with 167, 128, 111, 210, 183, and 201 in the respective domains. During preparation, writing training prompts were filtered against evaluation prompts using $n$-gram overlap with $n=3$. A training prompt was removed if its shared trigrams accounted for at least 10\% of the trigrams in any compared evaluation prompt. Normalization uses the evaluation prompt, making the overlap criterion directional. An additional exact-text comparison of the final training file against all WritingBench prompts finds no matches. Domain means weight these six groups equally, independently of their sizes.

\paragraph{Optimization and checkpoint selection.}
For both targets, all drafters use ten training epochs with AdamW, a learning rate of $6\times10^{-4}$, 4\% warmup, bfloat16 computation, and global batch size 64. Drafter depth follows Section~\ref{sec:setup}: one layer for EAGLE-3 and five for DFlash, DSpark, and \method. Results use the final checkpoint. \method\ retains the reference losses and adds the two auxiliary terms with the coefficients specified in Appendix~\ref{app:implementation}. These settings are also retained in DSpark + Tree and the single-loss ablations; the compared tree policy and serving implementation remain fixed. Evaluation applies no input-length truncation and limits generated output to 2,048 tokens; the 2,048-token prompt filter is used only when constructing the training mixture.

\paragraph{Comparison scope.}
The main comparisons evaluate the extension of parallel drafting under the common protocol in Section~\ref{sec:setup}. EAGLE-3 supplies an autoregressive-drafter reference in its evaluated chain configuration. DSpark + Tree is a reference-trained control with the same attention head dimension of 192 and the same inference policy as \method; it disables both tree losses. Single-loss ablations set only the corresponding coefficient to zero. The capacity ablation uses head dimension 128 with both objectives and the candidate policy retained.

\subsection{Metric Definitions and Aggregation}
\label{app:experimental-metrics}

\paragraph{Acceptance.}
For each dataset, acceptance length is the total number of newly committed tokens divided by the number of verification rounds. The anchor is excluded. The target bonus is included when emitted, while termination by an accepted end token or an output limit can shorten the final round. The eight-position block convention therefore corresponds to seven draft depths and at most eight newly committed tokens in an ordinary round. A tree budget of 32 counts non-root candidates; its target input has up to 33 positions including the anchor.

The position-wise metric at draft position $j$ divides the number of rounds that accept at least $j+1$ draft tokens by the number that propose that position. It is cumulative over accepted depth. The left panel of Figure~\ref{fig:acceptance} averages recorded rates at $C=12$ across the nine general benchmarks; shading shows their range across tasks. These rates and aggregate acceptance lengths come from separate evaluation records. The right panel therefore uses the directly recorded lengths in Table~\ref{tab:acceptance-length}, including their termination behavior. Figure~\ref{fig:summary} uses the same means of 4.23 for DSpark and 5.25 for \method, together with the suite-level DSpark + Tree aggregate of 4.93. Benchmark-level lengths and position-wise rates are not reported for this control. Acceptance length is measured in tokens per verification round, separately from throughput in tokens per second.

\paragraph{Serving throughput.}
Single-request throughput divides total output tokens by accumulated decoding time. The concurrency experiment uses a closed loop that admits a new request whenever a slot becomes free. System throughput divides all generated tokens by the elapsed workload time. For multiple cards, output counts are summed across workers and elapsed time is the maximum worker duration. At total concurrency 36, three workers each maintain up to 12 active requests. Each worker hosts its own model replica; the comparison measures aggregate serving throughput across these replicas.

Per-request throughput is computed from each request after its first output token: subsequent output tokens are divided by the interval from first-token emission to completion, and the resulting rates are averaged across requests. It excludes time to first token and need not equal system throughput divided by nominal concurrency. Prefill, the filling and draining of active slots, and variable completion times affect the two measurements differently.

\paragraph{Averages and relative gains.}
Within each target and workload group, Mean is the unweighted arithmetic mean of the benchmark metrics. In throughput tables, the Mean entry in a $\Delta$ row averages the relative improvement on each benchmark. In Tables~\ref{tab:acceptance-length} and~\ref{tab:writing-acceptance}, the corresponding entry instead compares the two mean acceptance lengths. Figure~\ref{fig:summary} displays arithmetic mean absolute metrics. Gains between tree configurations compare these means: this convention applies to the 4.2\% and 5.0\% throughput gains over DSpark + Tree and the 2.0\% system and 2.1\% per-request gains from increasing head dimension. The approximate 6.4\% acceptance gain uses the unrounded \method\ mean and the available DSpark + Tree aggregate of 4.93. Other aggregates are calculated from benchmark-level observations before rounding for presentation.

Throughput values are averaged over three repeated measurements before benchmark aggregation. We report these means without error bars or significance tests; uncertainty across repetitions is not quantified.

\subsection{Acceptance on WritingBench}
\label{app:writing-results}

\begin{table}[t]
\caption{Acceptance length on the six WritingBench domains, supplementing Table~\ref{tab:acceptance-length}. The target bonus is included when emitted. Mean weights domains equally; its $\Delta$ compares the two means for \method\ and DSpark. Other $\Delta$ entries compare individual domains. Bold marks column maxima.}
\label{tab:writing-acceptance}
\centering
\footnotesize
\setlength{\tabcolsep}{2pt}
\begin{tabular*}{\textwidth}{@{\extracolsep{\fill}}lccccccc@{}}
\toprule
\multicolumn{8}{@{}l}{\emph{Writing}} \\
\midrule
 & \hdr{Academic}{\& Engineering} & \hdr{Advertising}{\& Marketing} & Education & \hdr{Finance}{\& Business} & \hdr{Literature}{\& Arts} & \hdr{Politics}{\& Law} & Mean \\
\midrule
\multicolumn{8}{@{}l}{\emph{Qwen3-4B}} \\
EAGLE-3 & 2.09 & 1.98 & 2.06 & 2.19 & 1.94 & 2.09 & 2.06 \\
DFlash & 2.51 & 2.31 & 2.46 & 2.62 & 2.22 & 2.47 & 2.43 \\
\dspark{} & 2.79 & 2.51 & 2.72 & 2.90 & 2.41 & 2.75 & 2.68 \\
\ours{} & \textbf{3.68} & \textbf{3.27} & \textbf{3.58} & \textbf{3.84} & \textbf{3.16} & \textbf{3.60} & \textbf{3.52} \\
$\Delta$ (\%) & +31.9 & +30.3 & +31.6 & +32.4 & +31.1 & +30.9 & +31.4 \\
\midrule
\multicolumn{8}{@{}l}{\emph{Qwen3-8B}} \\
EAGLE-3 & 2.39 & 2.24 & 2.39 & 2.52 & 2.04 & 2.34 & 2.32 \\
DFlash & 2.41 & 2.26 & 2.40 & 2.54 & 2.10 & 2.37 & 2.35 \\
\dspark{} & 2.64 & 2.43 & 2.61 & 2.77 & 2.27 & 2.57 & 2.55 \\
\ours{} & \textbf{3.49} & \textbf{3.18} & \textbf{3.46} & \textbf{3.66} & \textbf{2.94} & \textbf{3.39} & \textbf{3.35} \\
$\Delta$ (\%) & +32.2 & +30.9 & +32.6 & +32.1 & +29.5 & +31.9 & +31.6 \\
\bottomrule
\end{tabular*}
\end{table}

Table~\ref{tab:writing-acceptance} supplements the general-task results in Table~\ref{tab:acceptance-length} with all six WritingBench domains. \method\ achieves the highest acceptance length in every domain for both targets. Relative gains over DSpark range from 30.3\% to 32.4\% for Qwen3-4B and from 29.5\% to 32.6\% for Qwen3-8B, showing that the higher mean acceptance extends across the writing suite. Together with the writing throughput results in Table~\ref{tab:throughput-c1}, these measurements establish that longer accepted paths translate into faster generation on these workloads. They evaluate the complete system; the contribution of tree supervision at fixed capacity and inference policy is examined separately in Section~\ref{sec:throughput-results}.

\subsection{Per-Request Throughput and Ablations}
\label{app:request-results}

\begin{table}[t]
\caption{Per-request decoding throughput on Qwen3-4B (tokens/s/request), supplementing Table~\ref{tab:concurrency}. Hardware and concurrency are identical. $\Delta$ gives relative gains over DSpark; its Mean entry averages per-benchmark gains.}
\label{tab:request-throughput}
\centering
\footnotesize
\setlength{\tabcolsep}{2pt}
\begin{tabular*}{\textwidth}{@{\extracolsep{\fill}}lcccccccccc@{}}
\toprule
\multicolumn{11}{@{}l}{\emph{Per-request throughput (tokens/s/req)}} \\
\midrule
 & GSM8K & \hdr{MATH}{500} & \hdr{AIME}{25} & \hdr{Human}{Eval} & MBPP & \hdr{LiveCode}{Bench} & \hdr{MT-}{Bench} & Alpaca & \hdr{Arena}{Hard-v2} & Mean \\
\midrule
\multicolumn{11}{@{}l}{\emph{$C=6$}} \\
Vanilla & 11.69 & \phantom{0}9.89 & \phantom{0}8.73 & 10.93 & 10.92 & \phantom{0}9.05 & 10.24 & 10.38 & \phantom{0}9.30 & 10.13 \\
\dspark{} & 40.44 & 38.58 & 34.25 & 31.59 & 29.51 & 32.56 & 18.61 & 18.51 & 20.46 & 29.39 \\
\ours{} & \textbf{44.09} & \textbf{44.42} & \textbf{38.24} & \textbf{36.13} & \textbf{33.84} & \textbf{36.36} & \textbf{23.10} & \textbf{22.05} & \textbf{24.46} & \textbf{33.63} \\
$\Delta$ (\%) & +9.0 & +15.1 & +11.6 & +14.4 & +14.7 & +11.7 & +24.1 & +19.1 & +19.6 & +15.5 \\
\midrule
\multicolumn{11}{@{}l}{\emph{$C=12$}} \\
Vanilla & 10.80 & \phantom{0}9.20 & \phantom{0}8.35 & \phantom{0}9.81 & \phantom{0}9.32 & \phantom{0}8.35 & \phantom{0}8.36 & \phantom{0}8.93 & \phantom{0}8.77 & \phantom{0}9.10 \\
\dspark{} & 29.52 & 33.21 & 29.94 & 24.81 & 22.60 & 25.15 & 15.85 & 15.42 & 17.38 & 23.76 \\
\ours{} & \textbf{32.99} & \textbf{38.44} & \textbf{36.76} & \textbf{29.50} & \textbf{26.24} & \textbf{28.52} & \textbf{19.73} & \textbf{18.66} & \textbf{19.87} & \textbf{27.86} \\
$\Delta$ (\%) & +11.8 & +15.7 & +22.8 & +18.9 & +16.1 & +13.4 & +24.5 & +21.0 & +14.3 & +17.6 \\
\midrule
\multicolumn{11}{@{}l}{\emph{$C=36$}} \\
Vanilla & 10.47 & \phantom{0}8.86 & \phantom{0}5.17 & \phantom{0}9.70 & \phantom{0}9.06 & \phantom{0}7.74 & \phantom{0}7.36 & \phantom{0}8.64 & \phantom{0}8.52 & \phantom{0}8.39 \\
\dspark{} & 31.46 & 31.97 & 19.21 & 24.17 & 21.97 & 24.44 & 13.48 & 14.80 & 17.20 & 22.08 \\
\ours{} & \textbf{33.61} & \textbf{35.82} & \textbf{21.10} & \textbf{27.22} & \textbf{24.39} & \textbf{26.51} & \textbf{15.58} & \textbf{17.60} & \textbf{19.09} & \textbf{24.55} \\
$\Delta$ (\%) & +6.8 & +12.0 & +9.8 & +12.6 & +11.0 & +8.5 & +15.6 & +18.9 & +11.0 & +11.8 \\
\bottomrule
\end{tabular*}
\end{table}

\begin{table}[t]
\caption{Per-request throughput for the ablations in Table~\ref{tab:ablation-c12} (tokens/s/request). Each variant retains tree inference. $\Delta$ compares each variant with DSpark; its Mean entry averages benchmark gains. Bold marks column maxima.}
\label{tab:request-ablation}
\centering
\footnotesize
\setlength{\tabcolsep}{2pt}
\begin{tabular*}{\textwidth}{@{\extracolsep{\fill}}lcccccccccc@{}}
\toprule
\multicolumn{11}{@{}l}{\emph{Per-request throughput (tokens/s/req)}} \\
\midrule
 & GSM8K & \hdr{MATH}{500} & \hdr{AIME}{25} & \hdr{Human}{Eval} & MBPP & \hdr{LiveCode}{Bench} & \hdr{MT-}{Bench} & Alpaca & \hdr{Arena}{Hard-v2} & Mean \\
\midrule
\dspark{} & 29.52 & 33.21 & 29.94 & 24.81 & 22.60 & 25.15 & 15.85 & 15.42 & 17.38 & 23.76 \\
\midrule
\ours{} & \textbf{32.99} & \textbf{38.44} & \textbf{36.76} & \textbf{29.50} & \textbf{26.24} & \textbf{28.52} & \textbf{19.73} & \textbf{18.66} & \textbf{19.87} & \textbf{27.86} \\
$\Delta$ (\%) & +11.8 & +15.7 & +22.8 & +18.9 & +16.1 & +13.4 & +24.5 & +21.0 & +14.3 & +17.6 \\
\midrule
w/o $\cL_{\mathrm{surv}}$ & 31.54 & 36.88 & 35.30 & 28.61 & 25.20 & 26.72 & 18.43 & 18.11 & 19.52 & 26.70 \\
$\Delta$ (\%) & \phantom{0}+6.8 & +11.1 & +17.9 & +15.3 & +11.5 & \phantom{0}+6.2 & +16.3 & +17.4 & +12.3 & +12.8 \\
\midrule
w/o $\cL_{\mathrm{cov}}$ & 32.42 & 38.24 & 34.45 & 28.69 & 25.88 & 28.10 & 19.15 & 18.26 & 19.81 & 27.22 \\
$\Delta$ (\%) & \phantom{0}+9.8 & +15.1 & +15.1 & +15.6 & +14.5 & +11.7 & +20.8 & +18.4 & +14.0 & +15.0 \\
\midrule
Head dim.\ 128 & 32.02 & 37.65 & 35.71 & 28.74 & 26.04 & 28.15 & 19.37 & 18.08 & 19.82 & 27.29 \\
$\Delta$ (\%) & \phantom{0}+8.5 & +13.4 & +19.3 & +15.8 & +15.2 & +11.9 & +22.2 & +17.3 & +14.0 & +15.3 \\
\bottomrule
\end{tabular*}
\end{table}

Table~\ref{tab:request-throughput} complements the system measurements in Table~\ref{tab:concurrency}. Mean per-request throughput increases from 29.39 to 33.63 tokens/s at $C=6$, from 23.76 to 27.86 at $C=12$, and from 22.08 to 24.55 at $C=36$. Every benchmark shows a higher average request decoding rate in each setting. This metric summarizes per-request throughput rather than tail latency.

The per-request ablations in Table~\ref{tab:request-ablation} support the component analysis in the main text. The full configuration improves throughput over DSpark by an average of 17.6\%. Removing survival reduces this gain to 12.8\%; removing coverage reduces it to 15.0\%; retaining head dimension 128 yields 15.3\%. At fixed objectives and tree policy, widening the head from 128 to 192 increases mean per-request throughput from 27.29 to 27.86 tokens/s, a 2.1\% gain computed before rounding. The corresponding mean system throughput increases by 2.0\%. The complete configuration has the highest per-request throughput on each benchmark in this ablation. Together with Table~\ref{tab:ablation-c12}, these results support combining tree supervision with the larger attention representation.

\subsection{Training Alignment across Benchmarks}
\label{app:detailed-figures}

Figure~\ref{fig:throughput-detail} expands the aggregate throughput panels of Figure~\ref{fig:summary}. Figures~\ref{fig:throughput-system-wide} and~\ref{fig:throughput-request-wide} provide larger views of the same system and per-request measurements. These figures share the underlying observations and are not additional experimental runs.

At $C=12$, mean system throughput is 254.68 tokens/s for DSpark, 283.54 for DSpark + Tree, and 295.34 for \method. At $C=36$, the corresponding means are 574.45, 615.68, and 646.44. Mean per-request rates follow the same ordering: 23.76, 26.66, and 27.86 at $C=12$, and 22.08, 23.84, and 24.55 at $C=36$. The acceptance panel in Figure~\ref{fig:summary} complements these throughput measurements: the two tree configurations achieve mean lengths of 4.93 and 5.25. Their difference of approximately 0.32 tokens per round indicates that tree supervision improves the use of a fixed verification budget. The comparison with DSpark measures the overall system improvement; the comparison between the two tree configurations holds capacity and inference policy fixed and isolates the change in training objectives.

The improvement from aligned supervision is positive in aggregate and varies by task. The reference-trained control has higher system throughput on AIME 2025, MBPP, and ArenaHard-v2 at $C=12$, and on GSM8K, MT-Bench, and ArenaHard-v2 at $C=36$. The complete system exceeds chain DSpark throughout both groups. Together, these results distinguish the benefit of the full training and inference extension from the additional gain due to supervision at fixed capacity and candidate policy.

\begin{figure}[t]
 \centering
 \includegraphics[width=\linewidth]{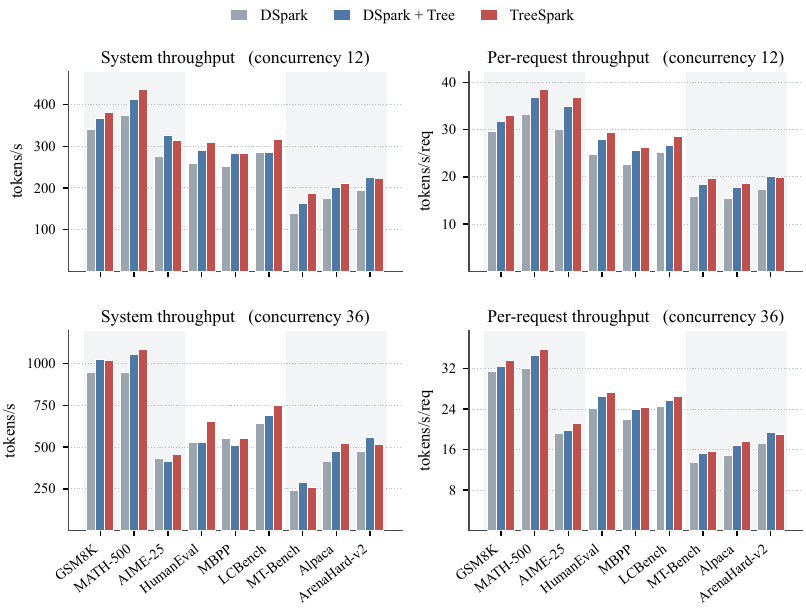}
 \caption{\textbf{Detailed throughput on Qwen3-4B.} System and per-request measurements at total concurrency 12 and 36 compare DSpark, DSpark + Tree, and \method. The two tree configurations use attention head dimension 192.}
 \label{fig:throughput-detail}
\end{figure}

\begin{figure}[t]
 \centering
 \includegraphics[width=\linewidth]{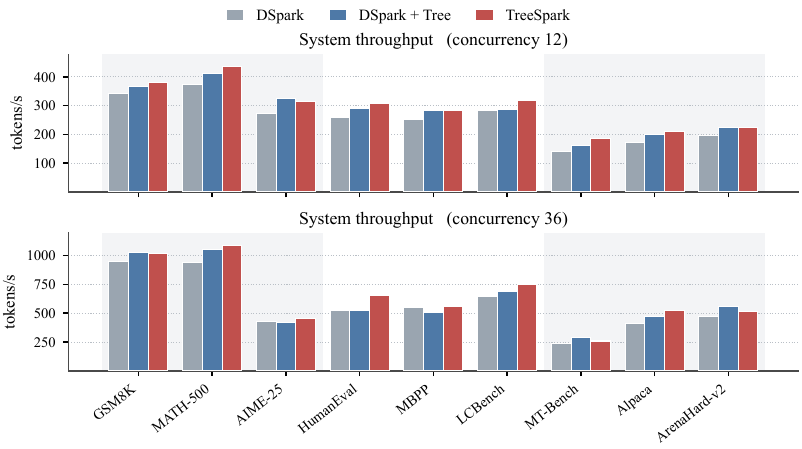}
 \caption{\textbf{System throughput by benchmark.} Enlarged view of the system measurements in Figure~\ref{fig:throughput-detail} under the serving configurations in Section~\ref{sec:setup}.}
 \label{fig:throughput-system-wide}
\end{figure}

\begin{figure}[t]
 \centering
 \includegraphics[width=\linewidth]{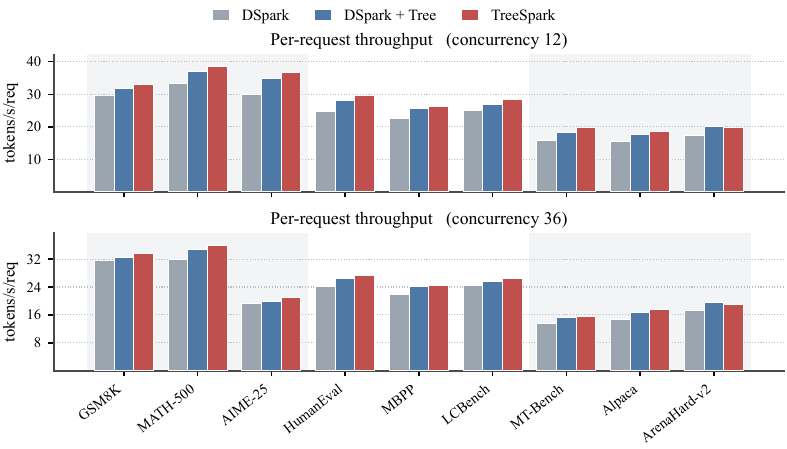}
 \caption{\textbf{Per-request throughput by benchmark.} Enlarged view of the per-request measurements in Figure~\ref{fig:throughput-detail}, averaged over request decoding intervals after the first output token.}
 \label{fig:throughput-request-wide}
\end{figure}

\paragraph{Evaluation scope.}
The experiments assess the proposed extension of parallel drafting with two Qwen3 target sizes, greedy decoding, and a fixed candidate budget. Comparisons with DFlash and DSpark measure complete-system efficiency; the reference-trained branching control and ablations assess the contribution of training alignment and capacity. Greedy output preservation is established under the verification assumptions in Appendix~\ref{app:verification}. Generalization across budgets and stochastic sampling remains to be evaluated.

\section{AI Assistance Disclosure}

During the preparation of this manuscript, we used AI as an assistive tool for software development and academic writing. Specifically, it was used to assist with the implementation of parts of the proposed method and experimental pipeline, including drafting, revising, and debugging code. It was also used to assist with drafting and editing portions of the manuscript and to improve the clarity, readability, grammatical correctness, and overall quality of the English writing. All AI-generated or AI-assisted code was carefully reviewed, tested, and validated by the authors, and all AI-assisted text was reviewed and revised before inclusion in the manuscript.

Generative AI tools were not used to independently formulate the research questions or hypotheses, determine the research methodology or experimental design, interpret the experimental results, or derive the scientific conclusions of this work. All methodological decisions, experimental analyses, and scientific claims were made and verified by the authors. We take full responsibility for the correctness, originality, and final content of this work, including any text, code, or other material produced with the assistance of generative AI.

\end{document}